\def\year{2021}\relax
\documentclass[letterpaper]{article} 
\usepackage{aaai21}  
\usepackage{times}  
\usepackage{helvet} 
\usepackage{courier}  
\usepackage[hyphens]{url}  
\usepackage{graphicx} 
\usepackage{natbib}  
\usepackage{caption} 
\usepackage[utf8]{inputenc} 
\usepackage[T1]{fontenc}    
\usepackage{url}            
\usepackage{booktabs}       
\usepackage{amsfonts}       
\usepackage{nicefrac}       
\usepackage{microtype}      
\usepackage{mathtools}
\usepackage{amsmath, amssymb}
\usepackage[capitalize]{cleveref}
\usepackage{algorithm}
\usepackage{algorithmic}
\usepackage{color}
\usepackage{amsthm}
\usepackage{subcaption}
\usepackage[toc,page]{appendix}

\newcommand{\data}{x}
\newcommand{\Data}{X}

\newcommand{\amat}{\boldsymbol{A}}
\newcommand{\labelvec}{\boldsymbol{y}}
\newcommand{\labelest}{\boldsymbol{\Tilde{y}}}
\newcommand{\labelestmat}{\boldsymbol{\Tilde{Y}}}
\newcommand{\numdata}{n}
\newcommand{\numweak}{m}
\newcommand{\numclasses}{K}
\newcommand{\weakvec}{\boldsymbol{w}}
\newcommand{\bounds}{\boldsymbol{\epsilon}}
\newcommand{\bound}{\epsilon}

\newcommand{\zvec}{\boldsymbol{z}}
\newcommand{\pvec}{\boldsymbol{p}}
\newcommand{\placeholder}{\boldsymbol{v}}
\newcommand{\wmat}{\boldsymbol{W}}
\newcommand{\umat}{\boldsymbol{U}}
\newcommand{\vmat}{\boldsymbol{V}}

\DeclareMathOperator{\proj}{\mathbb{P}}

\newtheorem{theorem}{Theorem}

\title{Constrained Labeling for Weakly Supervised Learning}

\author {
        Chidubem Arachie\textsuperscript{\rm 1}, Bert Huang\textsuperscript{\rm 2} \\
}
\affiliations {
    \textsuperscript{\rm 1}Department of Computer Science, Virginia Tech \\
    \textsuperscript{\rm 2}Department of Computer Science and
    Data Intensive Studies Center, Tufts University \\
    achid17@vt.edu, bert@cs.tufts.edu
}

\begin{document}

\maketitle

\begin{abstract}
Curation of large fully supervised datasets has become one of the major roadblocks for machine learning. Weak supervision provides an alternative to supervised learning by training with cheap, noisy, and possibly correlated labeling functions from varying sources. The key challenge in weakly supervised learning is combining the different weak supervision signals while navigating misleading correlations in their errors. In this paper, we propose a simple data-free approach for combining weak supervision signals by defining a constrained space for the possible labels of the weak signals and training with a random labeling within this constrained space. Our method is efficient and stable, converging after a few iterations of gradient descent. We prove theoretical conditions under which the worst-case error of the randomized label decreases with the rank of the linear constraints. 
We show experimentally that our method outperforms other weak supervision methods on various text- and image-classification tasks.
\end{abstract}

\section{Introduction}

Recent successful demonstrations of machine learning have created an explosion of interest. The key driver of these successes is the progress in deep learning. Researchers in different fields and industries are applying deep learning to their work with varying degrees of success. Training deep learning models typically requires massive amounts of data, and in most cases this data needs to be labeled for supervised learning. The process of collecting labels for large training datasets is often expensive and can be a major bottleneck for practical machine learning.

To enable machine learning when labeled data is not available, researchers are increasingly turning to weak supervision. Weakly supervised learning involves training models using noisy labels. Using multiple sources or forms of weak supervision is common, as it provides diverse information to the model. However, each source of weak supervision has its own bias that can be transmitted to the model. Different weak supervision signals can also conflict, overlap, or---in the worst case---make dependent errors. Thus, a naive combination of these weak signals would hurt the quality of a learned model. The key problem then is how to reliably combine various sources of weak signals to train an accurate model.

To solve this problem, we propose \emph{constrained label learning} (CLL), a method that processes various weak supervision signals and combines them to produce high-quality training labels. The idea behind CLL is that, given the weak supervision, we can define a constrained space for the labels of the unlabeled examples. The space will contain the true labels of the data, and any other label sampled from the space should be sufficient to train a model. We construct this space using the expected error of the weak supervision signals, and then we select a random vector from this space to use as training labels. Our analysis shows that, the space of labels considered by CLL improves to be tighter around the true labels as we include more information in the weak signals and that CLL is not confounded by redundant weak signals.

CLL takes as input (1) a set of unlabeled data examples, (2) multiple weak supervision signals that label a subset of data and can abstain from labeling the rest, and (3) a corresponding set of expected error rates for the weak supervision signals. While the weak supervision signals can abstain on various examples, we require that the combination of the weak signals have full coverage on the training data. The expected error rates can be estimated if the weak supervision signals have been tested on historical data or a domain expert has knowledge about their performance. In cases where the expected error rates are unavailable, they can be treated as a hyperparameter. 
Our experiments in \cref{sec:bounds} show that CLL is still effective when it is trained with a loose estimate of the weak signals.
Alternatively, we provide guidelines on how error rates can be estimated. 

We implement CLL as a stable, quickly converging, convex optimization over the candidate labels. CLL thus scales much better than many other weak supervision methods.
We show in \cref{sec:experiments} experiments that compare the performance of CLL to other weak supervision methods. On a synthetic dataset, CLL trained with a constant error rate is only a few percentage points from matching the performance of supervised learning on a test set. On real text and image classification tasks, CLL achieves superior performance over existing weak supervision methods on test data.

\section{Related Work}
\label{related}

Weakly supervised learning has gained prominence in recent years due to the need to train models without access to manually labeled data. The recent success of deep learning has exacerbated the need for large-scale data annotation, which can be prohibitively expensive. One weakly supervised paradigm, data programming, allows users to define \emph{labeling functions} that noisily label a set of unlabeled data \cite{bach2018snorkel,ratner2017snorkel,ratner2016data}. Data programming then combines the noisy labels to form probabilistic labels for the data by using a generative model to estimate the accuracies and dependencies of the noisy/weak supervision signals. This approach underlies the popular software package \emph{Snorkel} \cite{ratner2017snorkel}. Our method is related to this approach in that we use different weak signal sources and compile them into a single (soft) labeling. However, unlike Snorkel's methods, we do not train a generative model and avoid the need for probabilistic modeling assumptions. Recently, Snorkel MeTaL was proposed for solving multi-task learning problems with hierarchical structure \cite{ratner2018snorkel}. A user provides weak supervision for the hierarchy of tasks which is then combined in an end-to-end framework.

Another recently developed approach for weakly supervised learning is adversarial label learning (ALL) \cite{arachie2019adversarial}. ALL was developed for training binary classifiers from weak supervision. ALL trains a model to perform well in the worst case for the weak supervision by simultaneously optimizing model parameters and adversarial labels for the training data in order to satisfy the constraint that the error of the weak signals on the adversarial labels be within provided error bounds. The authors also recently proposed Stoch-GALL \cite{arachie2019adaptable}, an extension for multi-class classification that incorporates precision bounds. Our work is related to ALL and Stoch-GALL in that we use the same error definition the authors introduced. However, the expected errors we use do not serve as upper bound constraints for the weak signals. Additionally, CLL avoids the adversarial setting that requires unstable simultaneous optimization of the estimated labels and the model parameters.
Lastly, while ALL and Stoch-GALL require weak supervision signals to label every example, we allow for weak supervision signals that abstain on different data subsets. 

Crowdsourcing has become relevant to machine learning practitioners as it provide a means to train machine learning models using labels collected from different crowd workers \cite{carpenter2008multilevel,gao2011harnessing, karger2011iterative,khetan2017learning,liu2012variational,platanios2020learning,zhou2015regularized,zhou2016crowdsourcing}. The key machine learning challenge when crowdsourcing is to effectively combine the different labels obtained from human annotators. Our work is similar in that we try to combine different weak labels. However, unlike most methods for crowdsourcing, we cannot assume that the labels are independent of each other. Instead, we train the model to learn while accounting for dependencies between the various weak supervision signals. 

Ensemble methods such as boosting \cite{schapire2002incorporating} combine different weak learners (low-cost, low-powered classifiers) to create classifiers that outperform the various weak learners. These weak learners are not weak in the same sense as weak supervision. These strategies are defined for fully supervised settings. Although recent work has proposed leveraging unlabeled data to improve the accuracies of boosting methods \cite{balsubramani2015scalable}, our settings differs since we do not expect to have access to labeled data.

A growing set of weakly supervised applications includes web knowledge extraction ~\cite{bunescu:acl07,hoffman:acl11,mintz:acl09,riedel:ecml10,yao:emnlp10}, visual image segmentation~\cite{chen:cvpr14,xu:cvpr14}, and tagging of medical conditions from health records~\cite{halpern:health16}. As better weakly supervised methods are developed, this set will expand to include other important applications.

We will show an estimation method that is connected to those developed to estimate the error of classifiers without labeled data \cite{dawid1979maximum,jaffe2016unsupervised,madani2005co,platanios2014estimating, platanios2016estimating,steinhardt2016unsupervised}. These methods rely on statistical relationships between the error rates of different classifiers or weak signals. Unlike these methods, we show in our experiments that we can train models even when we do not learn the error rates of classifiers. We show that using a maximum error estimate of the weak signals, CLL learns to accurately classify.

Like our approach, many other methods incorporate human knowledge or side information into a learning objective. These methods, including posterior regularization \cite{druck2008learning} and generalized expectation (GE) criteria and its variants \cite{mann2008generalized,mann2010generalized}, can be used for semi- and weakly supervised learning. They work by providing parameter estimates as constraints to the objective function of the model so that the label distribution of the trained model tries to match the constraints. In our approach, we incorporate human knowledge as error estimates into our algorithm. However, we do not use the constraints for model training. Instead, we use them to generate training labels that satisfy the constraints, and these labels can then be used downstream to train any model.

\section{Constrained Label Learning}
\label{sec:method}

The goal of \emph{constrained label learning} (CLL) is to return accurate training labels for the data given the weak supervision signals. The estimation of these labels should be aware of the correlation among the weak supervision signals and should not be confounded by it. Toward this goal, we use the weak signals' expected error to define a constrained space of possible labelings for the data. Any vector sampled from this space can then be used as training labels. We consider the setting in which the learner has access to a training set of unlabeled examples, and a set of weak supervision signals from various sources that provide approximate indicators of the target classification for the data. Along with the weak supervision signals, we are provided estimates of the expected error rates of the weak signals. Formally, let the data be $\Data = [\data_1, \ldots, \data_\numdata ]$. These examples have corresponding labels $\labelvec = [y_1, \ldots, y_\numdata] \in \{0, 1\}^\numdata$. For multi-label classification, where each example may be labeled as a member of $\numclasses$ classes, we expand the label vector to include an entry for each example-class combination, i.e., $\labelvec = [y_{(1, 1)}, \ldots, y_{(n, 1)}, y_{(1, 2)}, \ldots, y_{(\numdata - 1, \numclasses)}, y_{(\numdata, \numclasses)}]$, where $y_{ij}$ is the indicator of whether the $i$th example is in class $j$.\footnote{We represent the labels as a vector for later notational convenience, even though it may be more naturally arranged as a matrix.} See \cref{fig:diagram} for an illustration of this arrangement.

\begin{figure}[tb]
\centering
\includegraphics[width=0.95\columnwidth]{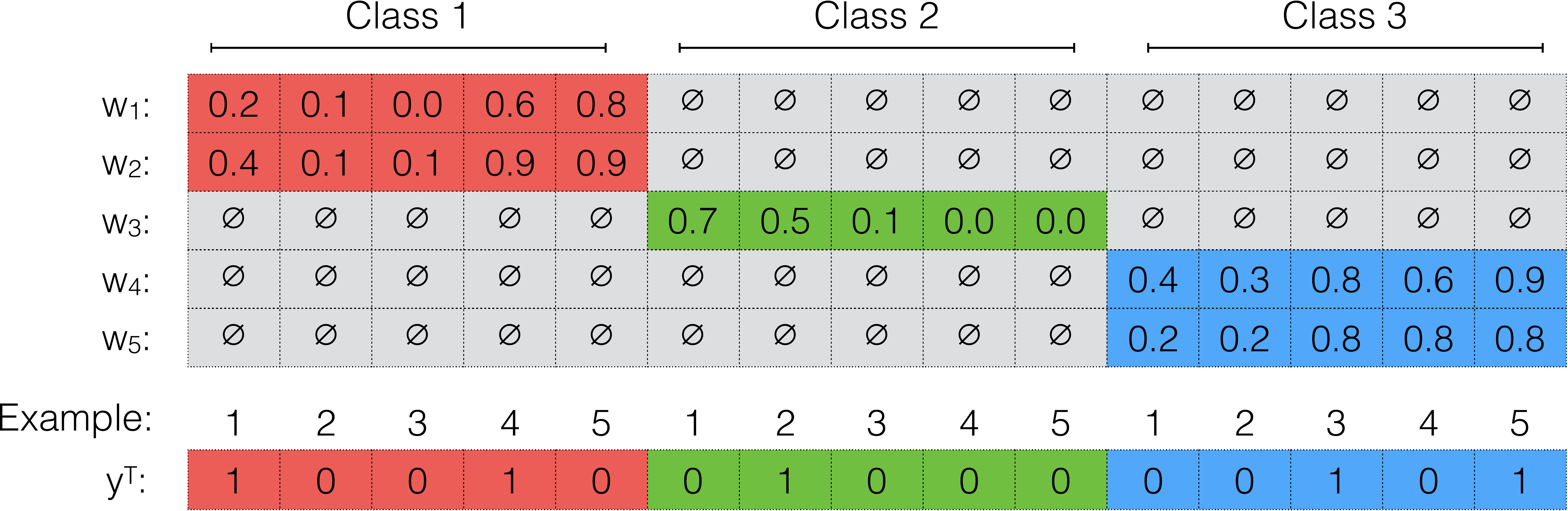}
\caption{\footnotesize Illustration of weak signals and label vectorized structure. For multi-class problems, we arrange the label vector so that it contains indicators for each example belonging to each class. The weak signals use the same indexing scheme. In this illustration, weak signals $\weakvec_1$ and $\weakvec_2$ estimate the probability of each example belonging to class 1 and abstain on estimating membership in all other classes.}
\label{fig:diagram}
\end{figure}

With weak supervision, the training labels $\labelvec$ are unavailable. Instead, we have access to $\numweak$ weak supervision signals $\{\weakvec_1, \ldots, \weakvec_\numweak\}$, where each weak signal $\weakvec \in [\emptyset, 0, 1]^n$ is represented as a vector of indicators that each example is in each class. The weak signals can choose to abstain on some examples. In that case, they assign a null value $\emptyset$ to that example's entry.  In practice, weak signals for multi-class problems typically only label one class at a time, such as a one-versus-rest classification rule, so they effectively abstain on all out-of-class entries. The weak signals can be soft labels (probabilities) or hard labels (class assignments) of the data. In conjunction with the weak signals, the learner also receives the expected error rates of the weak signals $\bounds = [\bound_1, \ldots, \bound_\numweak]$. In practice, the error rates of the weak signals are estimated or treated as a hyperparameter. The expected empirical error of a weak signal $\weakvec_i$ is
\begin{equation}
\begin{aligned}
\bound_i &= \frac{1}{n_i} \left( \mathbf{1}_{(\weakvec \neq \emptyset)} \weakvec_i^\top (1 - \labelvec_k) + \mathbf{1}_{(\weakvec \neq \emptyset)} (1 - \weakvec_i)^\top \labelvec_k \right)\\
&= \frac{1}{n_i} \left( \mathbf{1}_{(\weakvec \neq \emptyset)}(1 - 2 \weakvec_i)^\top \labelvec_k + \weakvec_i^\top \mathbf{1}_{(\weakvec \neq \emptyset)} \right),
\end{aligned}
\label{eq:error}
\end{equation}
where $\labelvec_k$ is the true label for the class $k$ that the weak signal $\weakvec_i$\ labels, $n_i = \sum \mathbf{1}_{(\weakvec_i \neq \emptyset)}$ and $\mathbf{1}_{(\weakvec_i \neq \emptyset)}$ is an indicator function that returns $1$ on examples the weak signals label (i.e., do not abstain on). Hence, we only calculate the error of the weak signals on the examples they label.

Analogously to \cref{eq:error}, we can express the expected error of all weak signals for the label vector as a system of linear equations in the form $\amat \labelvec = \boldsymbol{c}$. To do this, we define each row in $\amat$ as 
\begin{equation}
    \amat_i = \mathbf{1}_{(\weakvec_i \neq \emptyset)}(1 - 2 \weakvec_i),
    \label{eq:linearsystemA}
\end{equation}
a linear transformation of a weak signal $\weakvec$. Each entry in the vector $\boldsymbol{c}$ is the difference between the expected error of the weak signal and the sum of the weak signal, i.e., 
\begin{equation}
    \boldsymbol{c_i} = \numdata_i \bound_i - \weakvec_i^\top \mathbf{1}_{(\weakvec \neq \emptyset)}.
    \label{eq:linearsystemc}
\end{equation}
Valid label vectors then must be in the space
\begin{equation}
\begin{aligned}
\{\labelest | \amat \labelest = \boldsymbol{c} \wedge \labelest \in [0,1]^\numdata \} ~.
\end{aligned}
\label{eq:labelspace}
\end{equation}

The true label $\labelvec$ is not known. Thus, we want to find training labels $\labelest$ that satisfy the system of linear equations.

\subsection{Algorithm}

Having defined the space of possible labelings for the data given the weak signals, we explain here how we efficiently sample a vector of training labels from the space. First, we initialize a random $\labelest$ from a uniform distribution $\labelestmat \sim U(0,1)^\numdata$. Then we minimize a quadratic penalty on violations of the constraints defining the space. The objective function is
\begin{equation}
\begin{aligned}
\min_{\labelest \in [0, 1]^\numdata} ~~ \left\lVert \amat \labelest - \boldsymbol{c} \right\rVert_2^2 ~. 
\end{aligned}
\label{eq:mainobjective}
\end{equation}

The solution to this quadratic objective function gives us feasible labels for the training data. In our experiments, we estimate the error rates $\bounds$ of the weak signals. In cases where the error estimates make an infeasible space, this quadratic penalty acts as a squared slack. We solve \cref{eq:mainobjective} iteratively using projected Adagrad \cite{duchi2011adaptive}, clipping $\labelest$ values to $[0,1]^n$ between gradient updates. This approach is fast and efficient, even for large datasets. Our algorithm is a simple quadratic convex optimization that converges to a unique optimum for each initialization of $\labelest$. In our experiments, it converges after only a few iterations of gradient descent. We run the algorithm 3 times with random initialization of $\labelest$ and take the mean of the $\labelest$s as the estimated label. We observed that the labels returned from the different runs are very similar. We fix the number of iterations of gradient descent for each run to $200$ for all our experiments. The full algorithm is summarized in \cref{alg:CLL}.

\begin{algorithm}[tb]
\footnotesize
\begin{algorithmic}[1]
\REQUIRE Dataset $\Data = [\data_1, \ldots, \data_\numdata]$, weak signals $[\weakvec_1, \ldots, \weakvec_\numweak]$, and expected error $\bounds = [\bound_1, \ldots, \bound_\numweak]$ for the signals.
\STATE Define $\amat$ from \cref{eq:linearsystemA} and $\boldsymbol{c}$ from \cref{eq:linearsystemc} using the weak signals and expected errors.
\STATE Initialize $\labelest$ as $\labelest \sim U(0,1)^{\numdata}$
\WHILE{not converged}
\STATE Update $\labelest$ with its gradient from \cref{eq:mainobjective}
\STATE Clip $\labelest$ to $[0,1]^n$
\ENDWHILE \\
return estimated labels $\labelest$
\end{algorithmic}
\caption{Randomized Constrained Labeling}
\label{alg:CLL}
\end{algorithm}

\subsection{Analysis}
\label{sec:theory}

We start by analyzing the case where we have the true error $\bounds$, in which case the true label vector $\labelvec$ for CLL is a solution in the feasible space. Although the true error rates are not available in practice, this ideal setting is the motivating case for the CLL approach. To begin the analysis, consider an extreme case: if $\amat$ is a square matrix with full rank, then the only valid label $\labelest$ in the space is the true label, $\labelest = \labelvec$. Normally, $\amat$ is usually underdetermined, which means we have more data examples than weak signals. In this case, there are many solutions for $\labelest$, so we can analyze this space to understand how distant any feasible vector is from the vector of all incorrect labels. Since label vectors are constrained to be in the unit box, the farthest possible label vector from the true labels is $(1 - \labelvec)$. The result of our analysis is the following theorem, which addresses the binary classification case with non-abstaining weak signals.

\begin{theorem}
For any $\labelest \in [0, 1]^n$ such that $\amat \labelest = \boldsymbol{c}$, its Euclidean distance from the negated label vector $(1 - \labelvec) \in \{0, 1\}^\numdata$ is bounded below by
\begin{equation}
    ||\labelest - (1 - \labelvec)|| \ge \numdata ||\amat^+ (1 - 2 \bounds)||,
\end{equation}
where $\amat^+$ is the Moore-Penrose pseudoinverse of $\amat$.
\label{thm:distance}
\end{theorem}
\begin{proof}

We first relax the constrained space by removing the $[0, 1]^\numdata$ box constraints. We can then analyze the projection onto the feasible space:
\begin{equation}
\min_{\labelest} ||(1 - \labelvec) - \labelest|| ~~ \textrm{s.t.} ~~ \amat \labelest = \boldsymbol{c}.
\end{equation}
Define a vector $\zvec := \labelest - \labelvec$. We can rewrite the distance as
\begin{equation}
\min_{\zvec} ||(1 - 2\labelvec) - \zvec|| ~~ \textrm{s.t.} ~~ \amat \zvec = 0.
\end{equation}
The minimization is a projection of $(1 - 2y)$ onto the null space of $\amat$. Since the null and row spaces of a matrix are complementary, $(1 - 2 \labelvec)$ decomposes into
\[
(1 - 2\labelvec) = \proj_\textrm{row} (1 - 2 \labelvec) + \proj_\textrm{null}(1 - 2 \labelvec),
\]
where $\proj_\textrm{row}$ and $\proj_\textrm{null}$ are orthogonal projections into the row and null spaces of $\amat$, respectively. We can use this decomposition to rewrite the distance of interest:
\begin{equation}
    \begin{aligned}
        &||(1 - 2\labelvec) - \proj_\textrm{null}(1 - 2 \labelvec)||\\
        &= ||(1 - 2\labelvec) - ((1 - 2 \labelvec) -  \proj_\textrm{row}(1 - 2 \labelvec))||\\
        &= || \proj_\textrm{row}(1 - 2 \labelvec)||.
    \end{aligned}
\end{equation}
For any vector $\placeholder$, its projection into the row space of matrix $\amat$ is $\amat^+ \amat \placeholder$, where $\amat^+$ is the Moore-Penrose pseudoinverse of $\amat$. The distance of interest is thus $||\amat^+ \amat (1 - 2\labelvec)||$. We can use the definition of $\amat$ to further simplify. Let $\wmat$ be the matrix of weak signals $\wmat = [\weakvec_1, \ldots, \weakvec_\numweak]^\top$. Then the distance is
\begin{equation}
    \begin{aligned}
        &||A^+ (1 - 2\wmat) (1 - 2 \labelvec)||\\
        &= ||A^+ ((1 - 2\wmat) \vec{1}_n - 2 (1 - 2\wmat)   \labelvec)||\\
        &= ||A^+ (\numdata - 2\wmat \vec{1}_n - 2 \amat   \labelvec)||.
    \end{aligned}  
\end{equation}
Because $\amat \labelvec = \boldsymbol{c} = \numdata \bounds - \wmat \vec{1}_n$, terms cancel, yielding the bound in the theorem:
\begin{equation}
    \begin{aligned}
    &||A^+ (\numdata - 2\wmat \vec{1}_n - 2 \numdata \bounds + 2\wmat \vec{1}_n)||\\
    &=||A^+ (\numdata - 2 \numdata \bounds)|| = n||A^+ (1 - 2 \bounds)||.
    \end{aligned}  
\end{equation}
\end{proof}

This bound provides a quantity that is computable in practice. However, to gain an intuition about what factors affect its value, the distance formula can be further analyzed by using the singular-value decomposition (SVD) formula for the pseudoinverse. Consider SVD $\amat = \umat \Sigma \vmat^\top$. Then $\amat^+ = \vmat \Sigma^+ \umat^\top$, where the pseudoinverse $\Sigma^+$ contains the reciprocal of all nonzero singular values along the diagonal (and zeros elsewhere). The distance simplifies to
\begin{equation}
    \begin{aligned}
    n||\vmat \Sigma^+ \umat^\top(1- 2 \bounds)|| = n||\Sigma^+ \umat^\top(1 - 2 \bounds)||,
    \end{aligned}  
\end{equation}
since $\vmat$ is orthonormal. Furthermore, let $\pvec = \umat^\top(1 - 2 \bounds)$, i.e., $\pvec$ is a rotation of the centered error rates of the weak signals with the same norm as $(1 - 2\bounds)$. From this change of variables, we can decompose the distance into
\begin{equation}
    \begin{aligned}
    n||\Sigma^+ \pvec|| = n\sqrt{\sigma_1^2 p_1^2~ + \ldots +~ \sigma_m^2 p_m^2},
    \end{aligned}
    \label{eq:simplified}
\end{equation}
where $\sigma_j$ is the $j$th singular value of $\amat^+$.

As this distance grows toward $\sqrt{n}$, the space of possible labelings shrinks toward zero, at which point the only feasible label vectors are close to the true labels $\labelvec$.
\Cref{eq:simplified} indicates that the distance increases roughly as the rank of $\amat$ increases, in which case the number of non-zero singular values in $\Sigma^+$ increases, irrespective of how many actual weak signals are given. Thus, redundancy in the weak supervision does not affect the performance of CLL. The other key factor in the distance is how far from 0.5 the errors $\bounds$ are. These quantities can be interpreted as the diversity and number of the weak signals (corresponding to the rank) and their accuracies (the magnitude of $\pvec$).

Though the analysis is for length-$\numdata$ label vectors, it is straightforwardly extended to multi-label settings with length-$(\numdata \numclasses)$. And with careful indexing and tracking of the abstaining indicators, the same form of analysis can apply for abstaining weak signals.

\Cref{fig:rank} shows an empirical validation of \cref{thm:distance} on a synthetic experiment. 
We plot the error of the labels returned by CLL and majority voting as we change the rank of $\amat$.  We use a synthetic data for a binary classification task with 100 randomly generated examples containing 20 binary features. The weak signals are random binary predictions for the labels where each weak signal error rate is calculated using the true labels of the data. We start with 100 redundant weak signals by generating a matrix $\amat$ whose 100 columns contain copies of the same weak signal, giving it a rank of 1. We then iteratively increase the rank of $\amat$ by replacing copies of the weak signal with random vectors from the uniform distribution. The error of CLL labels approaches zero as the rank of the matrix increases while the majority vote error does not improve significantly.

\begin{figure}[tb]
\centering
\includegraphics[width=0.5\textwidth]{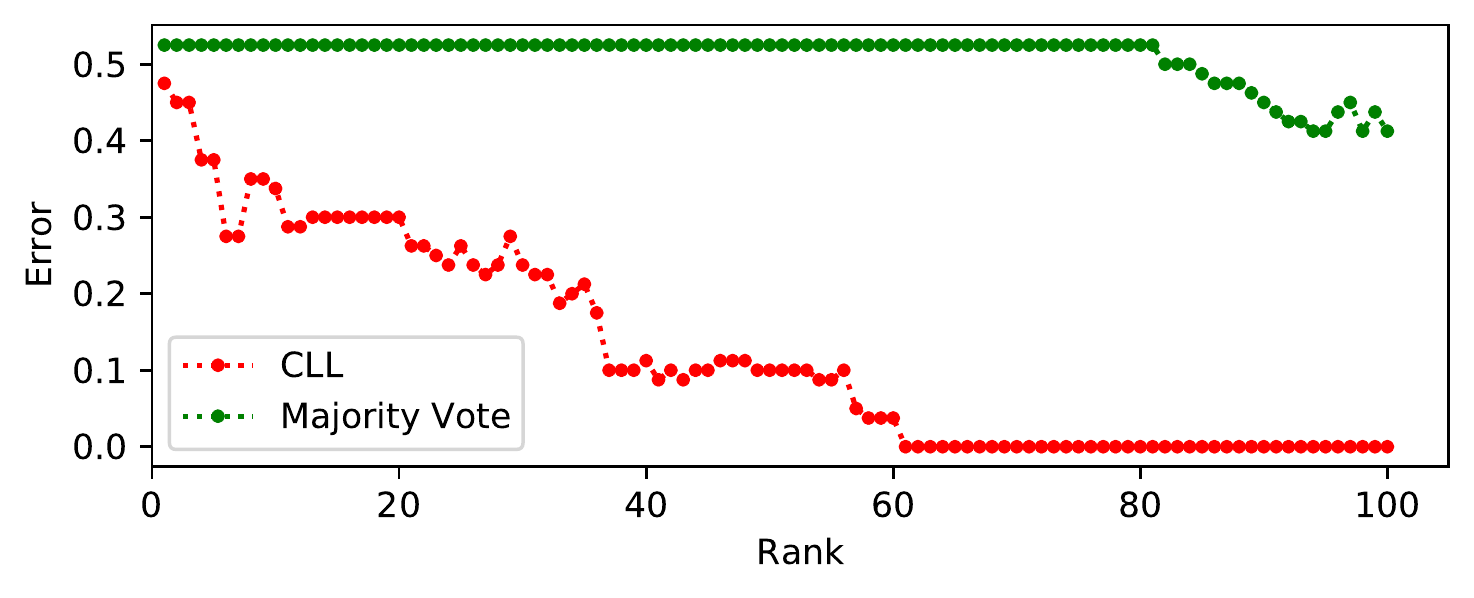}
\caption{\footnotesize Error of CLL estimated labels compared to majority vote as we increase the rank of $\amat$ by replacing redundant weak signals with linearly independent weak signals.}
\label{fig:rank}
\end{figure}

\subsection{Error Estimation}
\label{sec:bounds}

In our analysis, we assume that the expected error rates of the weak signals are available. This may be the case if the weak signals have been evaluated on historical data or if an expert provides the error rates. In practice, users typically define weak supervision signals whose error rates are unknown.
In this section, we discuss two approaches to handle such situations. We test these estimation techniques on real and synthetic data in our experiments, finding that CLL with these strategies forms a powerful weakly supervised approach.

\subsubsection{Agreement Rate Method}
\label{sec:agreement}

Estimating the error rates of binary classifiers using their agreement rates was first proposed by \citet{platanios2014estimating}. They propose two different objective functions for solving the error rates of classifiers using their agreement rates as constraints. Similar to MeTaL \cite{ratner2018snorkel}, we solve a matrix-completion problem to find a low-rank factorization for the weak signal accuracies. We assume that if the weak signals are conditionally independent, we can relate the disagreement rates to the weak signal accuracies. We implemented this method and report its performance in our synthetic experiment (see \cref{sec:experiments}). The one-vs-all form of the weak signals on our real datasets violates the assumption that each weak signal makes prediction on all the classes, so we cannot use the agreement rate method on our real data.

\subsubsection{Uniform Error Rate}
The idea of using uniform error rates of the weak signals was first proposed in ALL \cite{arachie2019adversarial}. Their experiments showed that ALL can learn as effectively as when using true error rates by using a constant for the error rates of all the weak signals on their binary classification datasets. We use this approach in our experiments and extend it to weak supervision signals that abstain and also on multi-class datasets. \Cref{fig:errorbounds} plots the accuracy of generated labels as we increase the error-rate parameter. On the binary-class SST-2 dataset, the label accuracy remains similar if the error rate is set between 0 and 0.5 and drops for values at least $0.5$. On the multiclass Fashion-MNIST data, we notice similar behavior where the label accuracies are similar between 0.05 and 0.1 and drop with larger values. We surmise that this behavior mirrors the type of weak supervision signals we use in our experiments. The weak signals in our real experiments are one-vs-all signals; hence a baseline signal (guessing 0 on all examples) will have an error rate of $\frac{1}{\numclasses}$. Performance deteriorates when the error rate is worse than this baseline rate.

\begin{figure}[tb]
    \centering
    \begin{subfigure}[b]{0.9\columnwidth}
    \centering
        \includegraphics[width=0.7\columnwidth]{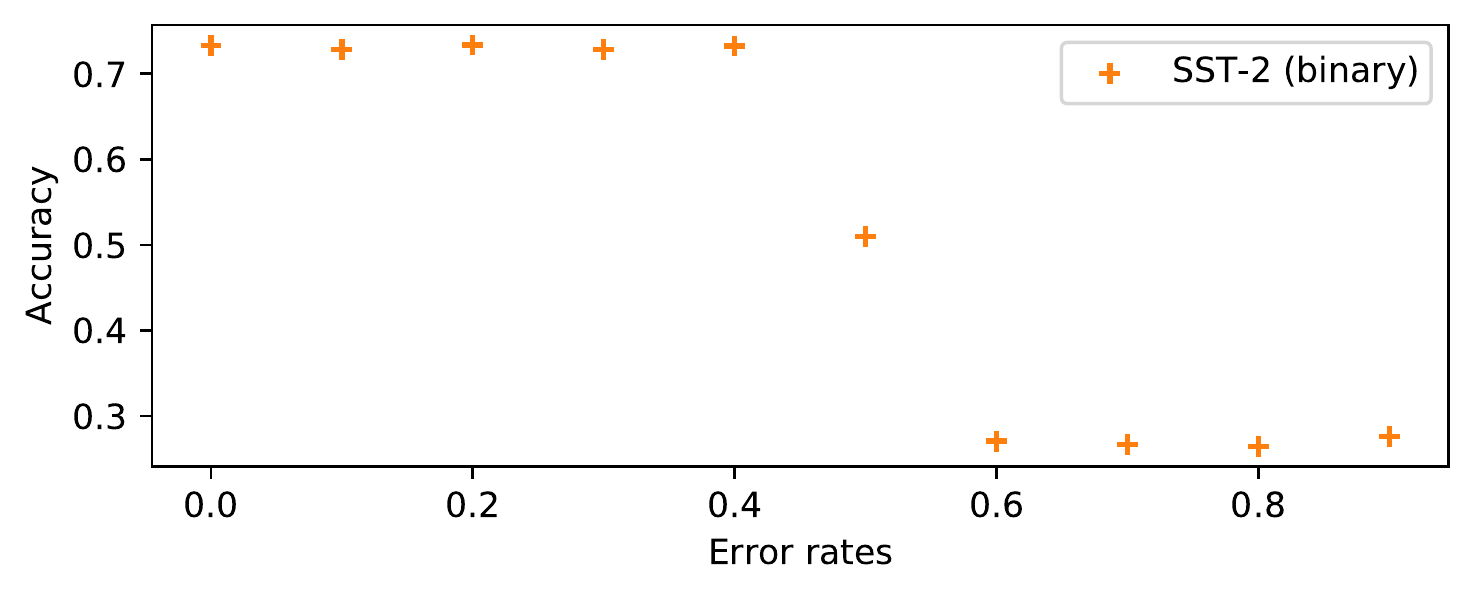}
        \label{fig:bc_error}
    \end{subfigure}
    ~ 
    \begin{subfigure}[b]{0.9\columnwidth}
    \centering
        \includegraphics[width=0.7\columnwidth]{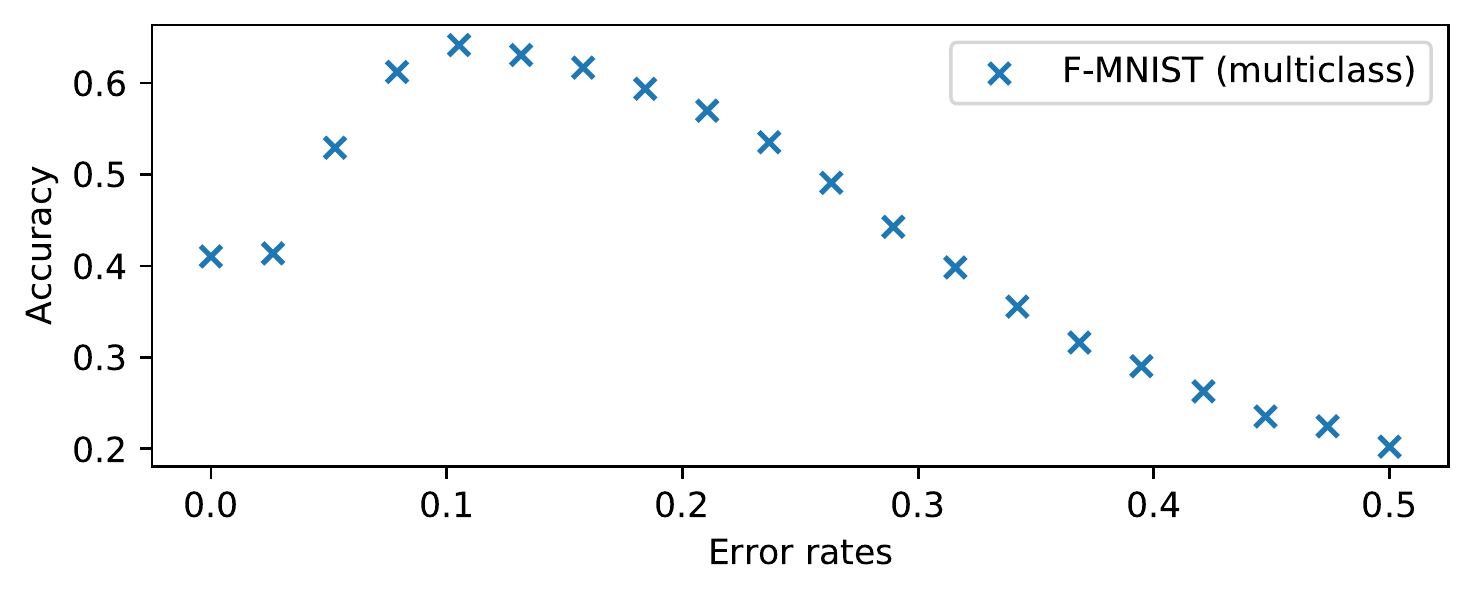}
        \label{fig:multi_error}
    \end{subfigure}
    \caption{\footnotesize Accuracy of constrained label learning as we increase the error rates from 0 to 1 on binary and 0 to 0.5 on multiclass datasets (SST-2 and Fashion-MNIST).}
\label{fig:errorbounds}
\end{figure}

\section{Experiments}
\label{sec:experiments}

We test constrained label learning on a variety of tasks on text and image classification. First, we measure the test accuracy of CLL on a synthetic dataset and compare its performance to that of supervised learning and other baselines. Second, we validate our approach on real datasets. 

For all our experiments, we compare CLL to other weakly supervised methods: data programming (DP) \cite{ratner2016data} and majority-vote (MV) or averaging (AVG). Additionally, on our real datasets we show comparison to regularized minimax conditional entropy for crowdsourcing (MMCE) \cite{zhou2015regularized}. For reference, we include the performance of supervised learning baseline. On the image datasets, we show comparison of CLL to Stoch-GALL, a multiclass extension of adversarial label learning. It is worth noting that DP was developed for binary classification, thus to compare its performance on our multiclass datasets, we run DP on the weak signals that label each class in the datasets. All the weak signals on the real datasets are one-vs-all signals meaning they only label a single class and abstain on other classes.

\subsection{Synthetic Experiment}

\begin{table}[tb]
\footnotesize
\centering
\begin{tabular}[t]{llc}
\toprule
Method & Test Accuracy \\
\midrule
CLL (Agr. rate $\bounds$) & \textbf{0.668}$\pm$ 0.005 \\
CLL (Constant $\bounds$)     & 0.630$\pm$ 0.009 \\
Data Programming  & 0.504$\pm$ 0.000\\
Majority Vote     & 0.504$\pm$ 0.000 \\\\
CLL (True $\bounds$)    & 0.675 $\pm$ 0.024 \\
Supervised Learning    & 0.997$\pm$ 0.001 \\
\bottomrule
\end{tabular}
\caption{\footnotesize Classification accuracies of the different methods on synthetic data using dependent weak signals. We report the mean and standard deviation over three trials.}
\label{tab:dependent}
\end{table}

\begin{table}[tb]
\footnotesize
\centering
\begin{tabular}[t]{llc}
\toprule
Method & Test Accuracy \\
\midrule
CLL (Agr. rate $\bounds$) & \textbf{0.984}$\pm$ 0.003 \\
CLL (Constant $\bounds$)     & 0.978$\pm$ 0.004 \\
Data Programming  & 0.978$\pm$ 0.003\\
Majority Vote     & 0.925$\pm$ 0.009 \\\\
CLL (True $\bounds$)    & 0.985$\pm$ 0.0004 \\
Supervised Learning    & 0.997$\pm$ 0.001 \\
\bottomrule
\end{tabular}
\caption{\footnotesize Classification accuracies of the different methods on synthetic data using independent weak signals. We report the mean and standard deviation over three trials}
\label{tab:independent}
\end{table}

\begin{table*}[tb]
\centering
\begin{tabular}{lcccc}
\toprule
Datasets & CLL & MMCE & DP  & MV \\
\midrule
IMDB  & \textbf{0.736}$\pm$ 0.0005 &0.573 & 0.693  & 0.702 \\
SST-2 & \textbf{0.678}$\pm$ 0.0004 &0.677 & 0.666 & 0.666 \\
YELP-2  & 0.765$\pm$ 0.0002 &0.685 & 0.770 & \textbf{0.775} \\
TREC-6 & 0.842$\pm$ 0.004 &0.833 & \textbf{0.898} & 0.273 \\
\bottomrule
\end{tabular}
\caption{\footnotesize Label accuracies of CLL compared to other weak supervision methods on different text classification datasets. We report the mean and standard deviation over three trials. CLL is trained using $\bounds$ = 0.01 on the text classification datasets.
}
\label{tab:label_accuracies}
\end{table*}

\begin{table*}[tb]
\centering
\begin{tabular}{lcccc|c}
\toprule
Datasets & CLL & MMCE & DP & MV & Supervised \\
\midrule
IMDB  & \textbf{0.740}$\pm$ 0.005 &0.551 & 0.623$\pm$ 0.007 & 0.724$\pm$0.004 & 0.820$\pm$0.003 \\
SST-2 & \textbf{0.729}$\pm$ 0.001 &0.727 & 0.720$\pm$ 0.001  & 0.720$\pm$ 0.0009 & 0.792$\pm$ 0.001 \\
YELP-2   & \textbf{0.840}$\pm$ 0.0007 &0.68 & 0.760$\pm$ 0.005 & 0.798$\pm$ 0.007 & 0.879$\pm$ 0.001 \\
TREC-6 & \textbf{0.641}$\pm$ 0.022 &0.64 & 0.627$\pm$ 0.014 & 0.605$\pm$ 0.006 & 0.700$\pm$ 0.024 \\
\bottomrule
\end{tabular}
\vspace{5pt}
\caption{\footnotesize Test accuracies of CLL compared to other weak supervision methods on different text classification datasets. We report the mean and standard deviation over three trials. CLL is trained using $\bounds$ = 0.01 on the text classification datasets }
\label{tab:test_accuracies}
\end{table*}

\begin{table*}[tb]
\centering
\begin{tabular}{lccccc}
\toprule
Datasets & CLL & MMCE & DP & AVG & Stoch-GALL \\
\midrule
SVHN   & \textbf{0.575}$\pm$ 0.001 & 0.1 & 0.42  & 0.444 & 0.196$\pm$ 0.025\\
Fashion-MNIST   & \textbf{0.658}$\pm$ 0.001 &0.147 & 0.65 & 0.649 & 0.488$\pm$ 0.002 \\
\bottomrule
\end{tabular}
\caption{\footnotesize Label accuracies of CLL compared to other weak supervision methods on image datasets. We report the mean and standard deviation over three trials. 
CLL is trained using $\bounds$ = $\frac{1}{\numclasses}$ on the datasets and it outperforms other baseline approaches.}
\label{tab:image_label_accuracies}
\end{table*}

\begin{table*}[t]
\centering
\begin{tabular}{lccccc|c}
\toprule
Datasets & CLL & MMCE & DP & AVG & Stoch-GALL & Supervised \\
\midrule
SVHN   & \textbf{0.670}$\pm$ 0.031 & 0.1 & 0.265$\pm$ 0.004 & 0.432$\pm$ 0.001 & 0.366$\pm$ 0.003 & 0.851$\pm$ 0.002\\
Fashion-MNIST  & \textbf{0.695}$\pm$ 0.002 &0.151 & 0.635$\pm$ 0.0004  & 0.666$\pm$ 0.002 & 0.598$\pm$ 0.002 & 0.852$\pm$ 0.003\\
\bottomrule
\end{tabular}
\caption{\footnotesize Test accuracies of CLL compared to other weak supervision methods on image datasets. We report the mean and standard deviation over three trials.  CLL is trained using $\bounds$ = $\frac{1}{\numclasses}$ on the datasets.}
\label{tab:image_test_accuracies}
\end{table*}

We construct a toy dataset for a binary classification task where the data has 200 randomly generated binary features and 20,000 examples, 16,000 for training and 4,000 for testing. Each feature vector has between 50\% to 70\% correlation with the true label. We define two scenarios for our synthetic experiments. We run the methods using (1) dependent weak signals and, (2) independent weak signals. In both experiments, we use $10$ weak signals that have at most $30\%$ coverage on the data and conflicts on their label assignments. The dependent weak signals were constructed by generating one weak signal that is copied noisily 9 times (randomly flipping $20\%$ of the labels). The original weak signal labeled $30\%$ of the data points and had an accuracy in $[0.5,0.6]$. So, on average, we expect to perturb $6\%$ of its labels on the copies. The independent weak signals are randomly generated to have accuracies in the range $[0.6,0.7]$.

We report in \cref{tab:dependent} and \cref{tab:independent}  the label and test accuracy from running CLL using true error rates for the weak signals, error rates estimated via agreement rate described in \cref{sec:agreement}, and error rates using a maximum error rate constant set to 0.4 as the expected error for all the weak signals. CLL trained using the true $\bounds$ obtains the highest test accuracy compared to the other baselines, and its performance almost matches that of supervised learning in \cref{tab:independent}. With the true bounds, CLL slightly outperforms CLL trained using estimated and constant $\bounds$. More interestingly, the results in \cref{tab:dependent} show that our method outperforms other baselines that are strongly affected by the dependence in the weak signals. The generative model of data programming assumes that the weak signals are independent given the true labels, but this is not the case in this setup as the weak signals are strongly dependent. Thus the conditional independence violation hurts its performance and essentially reduces it to performing a majority vote on the labels.

Since our evaluation in \cref{fig:errorbounds} demonstrated that CLL is not very sensitive to the choice of error rate, we set the error rates $\bounds$ = 0.01 on the text datasets and $\bounds$ = $\frac{1}{\numclasses}$ on the image datasets. We choose these values because our weak signals in the text dataset tend to label few examples and have low error rates thus we prefer not to under-constrain the optimization by using high error rates values for the one-vs-all weak-signals. In contrast, our human labeled weak signals on the image datasets have high error rates hence we set the error rate value to the baseline value for one-vs-all signals. 

\subsection{Real Experiments}

\begin{table}[tb]
\centering
\resizebox{\columnwidth}{!}{
\begin{tabular}{lccrr}
\toprule
Dataset & No. classes & No. weak signals & Train Size & Test Size \\
\midrule
IMDB  & 2 & 10 & 29,182 & 20,392 \\
SST-2 & 2 & 14 & 3,998 & 1,821 \\
YELP-2 & 2 & 14 & 45,370 & 10,000 \\
TREC-6 & 6 & 18 & 4,988 & 500 \\
SVHN   & 10 & 50 & 73,257 & 26,032 \\
Fashion-MNIST   & 10 & 50 & 60,000 & 10,000 \\
\bottomrule
\end{tabular}}
\caption{Summary of datasets, including the number of weak signals used for training.}
\label{tab:datasets}
\end{table}

The data sets for our real experiments and their weak signal generation process are described below. \Cref{tab:datasets} summarizes the key statistics about these datasets. Our code and datasets are provided here.\footnote{ \url{https://github.com/VTCSML/Constrained-Labeling-for-Weakly-Supervised-Learning}}

\noindent\textbf{IMDB} ~~
The IMDB dataset \cite{maas2011learning} is used for sentiment analysis. The data contains reviews of different movies, and the task is to classify user reviews as either positive or negative in sentiment. We provide weak supervision by measuring mentions of specific words in the movie reviews. We created a set of positive words that weakly indicate positive sentiment and negative words that weakly indicate negative sentiment. We chose these keywords by looking at samples of the reviews and selecting popular words used in them. Many reviews could contain both positive and negative keywords, and in these cases, the weak signals will conflict on their labels.  We split the dataset into training and testing subsets, where any example that contains one of our keywords is placed in the training set.
Thus, \emph{the test set consists of reviews that are not labeled by any weak signal}, making it important for the weakly supervised learning to generalize beyond the weak signals.
The dataset contains 50,000 reviews, of which 29,182 are used for training and 20,392 are test examples. 

\noindent\textbf{SST-2} ~~
The Stanford Sentiment Treebank (SST-2) is another sentiment analysis dataset \cite{socher2013recursive} containing movie reviews. Like the IMDB dataset, the goal is to classify reviews from users as having either positive or negative sentiment. We use similar keyword-based weak supervision but with different keywords. We use the standard train-test split provided by the original dataset. While the original training data contained 6,920 reviews, our weak signals only cover 3,998 examples. Thus, we used the reduced data size to train our model. We use the full test set of 1,821 reviews.

\noindent\textbf{YELP-2} ~~
We used the Yelp review dataset containing user reviews of businesses from the Yelp Dataset Challenge in 2015. Like the IMDB and SST-2 dataset, the goal is to classify reviews from users as having either positive or negative sentiment. We converted the star ratings in the dataset by considering reviews above 3 stars rating as positive and negative otherwise. We used similar weak supervision generating process as in SST-2. We sampled 50,000 reviews for training and 10,000 for testing from the original data set. Our weak signals only cover 45,370 data points, thus, we used the reduced data size to train our model.

\noindent\textbf{TREC-6} ~~
TREC is a question classification dataset consisting of fact-based questions divided into different categories \cite{li2002learning}. The task is to classify questions to predict what category the question belongs to. We use the six-class version (TREC-6) from which we use 4,988 examples for training and 500 for testing. The weak supervision we use combines word mentions with other heuristics we defined to analyze patterns of the question and assign a class label based on certain patterns.

\noindent\textbf{SVHN} ~~
The Street View House Numbers (SVHN) \cite{netzer2018street} dataset represents the task of recognizing digits on real images of house numbers taken by Google Street View. Each image is a $32 \times 32$ RGB vector. The dataset has 10 classes and has 73,257 training images and 26,032 test images. We define 50 weak signals for this dataset. For this image classification dataset, we augment 40 other human-annotated weak signals (four per class) with ten pseudolabel predictions of each class from a model trained on 1\% of the training data. The human-annotated weak signals are nearest-neighbor classifiers where a human annotator is asked to mark distinguishing features about an exemplar image belonging to a specific class.  We then calculate pairwise Euclidean distances between the pixels in the marked region across images. We convert the Euclidean scores to probabilities (soft labels for the examples) via a logistic transform. Through this process, an annotator is guiding the design of a simple one-versus-rest classifier, where images most similar to the reference image are more likely to belong to its class.

\noindent\textbf{Fashion-MNIST} ~~
The Fashion-MNIST dataset \cite{xiao2017fashion} represents the task of recognizing articles of clothing where each example is a $28 \times 28$ grayscale image. The images are categorized into 10 classes of clothing types where each class contains 6,000 training examples and 1,000 test examples. We used the same format of weak supervision signals as in the SVHN dataset (pseudolabels and human-annotated nearest-neighbor classifiers).

\noindent\textbf{Models} ~~ For the text analysis tasks, we use 300-dimensional GloVe vectors \cite{pennington2014glove} as features for the text classification tasks. Then we train a simple two-layer neural network with 512 hidden units and ReLU activation in its hidden layer. The model for the image classification tasks is a six-layer convolutional neural network model with a 3$\times$3 filter and 32 channels at each layer. We use a sigmoid function as the output layer for both models in our experiment. Thus we train using binary cross-entropy loss with the soft labels returned by CLL, which represent the probability of examples belonging to classes. 

\noindent\textbf{Results} ~~
\Cref{tab:label_accuracies,tab:test_accuracies} list the performance of the various weakly supervised methods on text classification datasets, while \cref{tab:image_label_accuracies,tab:image_test_accuracies} list the performance of various weakly supervised methods on image classification datasets. Considering both types of accuracy, CLL is able to output labels for the training data that train high-quality models for the test set. CLL outperforms all competing methods on test accuracy on the datasets. Interestingly, on Yelp and Trec-6 datasets, CLL label accuracy is lower than that of competing baselines yet CLL still achieves superior test accuracy. We surmise that CLL label accuracy is lower than competing methods on some datasets because of the inaccuracy in the error estimates. Generally, CLL is able to learn robust labels from the weak signals, and it seems to pass this information to the learning algorithm to help it generalize on unseen examples. For example, on the IMDB dataset, we used keyword-based weak signals that only occur on the training data. The model trained using CLL labels performs better on the test set than models trained with labels learned from data programming or majority vote. CLL outperforms all competing methods on the image classification tasks. On the digit recognition task (SVHN), CLL outperforms the best compared method (average) by over $13$ percentage points for the label accuracy and $23$ percentage points on the test data. CLL is able to better synthesize information from the low-quality human-annotated signals combined with the higher-quality pseudolabel signals.

\section{Conclusion}
We introduced constrained label learning (CLL), a weakly supervised learning method that combines different weak supervision signals to produce probabilistic training labels for the data. 
CLL defines a constrained space for the labels of the training data by requiring that the errors of the weak signals agree with the provided error estimates. 
CLL is fast and converges after a few iterations of gradient descent. Our theoretical analysis shows that the accuracy of our estimated labels increases as we add more linearly independent weak signals. 
This analysis is consistent with the intuition that the constrained-space interpretation of weak supervision avoids overcounting evidence when multiple redundant weak signals provide the same information, since they are linearly dependent.
Our experiments compare CLL against other weak supervision approaches on different text and image classification tasks. The results demonstrate that CLL outperforms these methods on most tasks. Interestingly, we are able to perform well when we train CLL using a worst case uniform error estimate for the weak signals. This shows that CLL is robust and not too sensitive to inaccuracy in the error estimates. In future work, we aim to theoretically analyze the behavior of this approach in such settings where the error rates are unreliable, with the hope that theoretical understanding will suggest new approaches that are even more robust.

\section*{Acknowledgments}
Arachie and Huang were both supported by a grant from the U.S. Department of Transportation, University Transportation Centers Program to the Safety through Disruption University Transportation Center (69A3551747115).

\bibliography{aaai}

\end{document}